\documentclass[10pt, twocolumn]{article}

\usepackage{silence}
    \usepackage[utf8]{inputenc}
    \usepackage[T1]{fontenc}
    \usepackage{siunitx}
    \usepackage{colonequals}
    \usepackage[super]{nth}
    \usepackage{algorithm}
    \usepackage{amsmath}
    \usepackage{amsthm}
    \usepackage{amssymb}
    \usepackage{tikz}
    \usepackage{url}
    \usepackage{float}
    \usepackage{tabularx}
    \usepackage{pbox}
    \usepackage{tabulary}
    \usepackage{breqn}
    \usepackage{longtable}
    \usepackage{mdframed}
    \usepackage{booktabs}
    \usepackage{physics}

    \usepackage{graphicx}
    \graphicspath{{./images/}{./}}

    \usepackage{wrapfig}
    \usepackage[font={small}]{caption}
    \usepackage{enumitem}
    \usepackage{physics}
    \usepackage{lipsum}
    \usepackage{mwe}
    \usepackage{breqn}
    \usepackage[noend]{algpseudocode}
    \usepackage{subcaption}
    \usepackage{graphicx}
    \usepackage{enumitem}
    \usepackage{bm}
    \usepackage{textcmds}

    \usepackage{apptools}
    \AtAppendix{\counterwithin{lemma}{section}}
    \AtAppendix{\counterwithin{remark}{section}}
    
    \usepackage{hyperref}
    \hypersetup{
        colorlinks,
            citecolor=black,
        filecolor=black,
        linkcolor=black,
        urlcolor=black
    }

\newtheorem{theorem}{Theorem}

\newtheorem{definition}{Definition}

\newtheorem{example}{Example}
\newtheorem{alg}{Algorithm}

\numberwithin{equation}{section} 
\numberwithin{lemma}{section} 
\numberwithin{theorem}{section} 
\numberwithin{definition}{section} 
\numberwithin{corollary}{section} 

\newcommand{\sigNorm}[1]{\lVert #1 \rVert_{\Sig^N(\mathcal C)}}
\newcommand{\Span}[1]{\text{span}\set{#1}}

\newcommand{\R}{\mathbb{R}}

\newcommand{\normTwo}[1]{\left\lVert#1\right\rVert_2}

\newcommand{\set}[1]{\left\{ #1 \right\}}

\newcommand{\squareBracket}[1]{\left[ #1 \right]}

\newcommand{\bracket}[1]{\left( #1\right)}

\usepackage[margin=1in]{geometry} 
\usepackage{multirow}

\newtheorem{thm}{Theorem}[section]
\newtheorem{proposition}[thm]{Proposition}

\theoremstyle{definition}

\usepackage[OT2,OT1]{fontenc}
\newcommand\cyr
{
\renewcommand\rmdefault{wncyr}
\renewcommand\sfdefault{wncyss}
\renewcommand\encodingdefault{OT2}
\normalfont
\selectfont
}
\DeclareTextFontCommand{\textcyr}{\cyr}
\DeclareSymbolFont{cyrletters}{OT2}{wncyss}{m}{n}

\DeclareMathOperator{\Sig}{\mathrm{Sig}}

\DeclareMathOperator{\dist}{\mathrm{dist}}

\newcommand{\oxford}{Mathematical Institute, University of Oxford}
\newcommand{\turing}{The Alan Turing Institute, London}

\title{Dimensionless Anomaly Detection on Multivariate Streams with Variance Norm and Path Signature}

\author{%
Zhen Shao\thanks{\oxford. (\texttt{shaoz@maths.ox.ac.uk})}, 
Ryan Sze-Yin Chan \thanks{\turing (\texttt{rchan@turing.ac.uk})},
Thomas Cochrane \thanks{\turing (\texttt{thomasc@turing.ac.uk})}, 
Peter Foster \thanks{\turing (\texttt{pfoster@turing.ac.uk})},
Terry Lyons  \thanks{\oxford. \turing (\texttt{terry.lyons@maths.ox.ac.uk})}
}

\begin{document}

\maketitle

\begin{abstract}
In this paper, we propose a dimensionless anomaly detection method for multivariate streams. Our method is independent of the unit of measurement for the different stream channels, therefore dimensionless. We first propose the variance norm, a generalisation of Mahalanobis distance to handle infinite-dimensional feature space and singular empirical covariance matrix rigorously. We then combine the variance norm with the path signature, an infinite collection of iterated integrals that provide global features of streams, to propose SigMahaKNN, a method for anomaly detection on (multivariate) streams. We show that SigMahaKNN is invariant to stream reparametrisation, stream concatenation and has a graded discrimination power depending on the truncation level of the path signature. We implement SigMahaKNN as an open-source software, and perform extensive numerical experiments, showing significantly improved anomaly detection on streams compared to isolation forest and local outlier factors in applications ranging from language analysis, hand-writing analysis, ship movement paths analysis and univariate time-series analysis.
\end{abstract}

\section{Introduction}
Anomaly detection is the semi-supervised learning problem where one is given a corpus of normal objects, with the aim of learning a map deciding whether a given object belongs to the corpus or not. Unlike in binary classification problems, where two types of objects are given in the training sample, here we only have the normal objects in the training set, and the anomaly detected does not necessarily belong to the same class, they are simply different from objects in the corpus. 

Streams are a ubiquitous data form found in many real-world situations. They are maps $\squareBracket{a, b} \to \R^d$. Examples include financial time series, the movement paths of objects/humans, and various measurement signals, e.g. ECG (Electrocardiogram) signals in medical devices or radio-astronomical signals from the telescope. This work focuses on detecting anomalous streams given a corpus of normal streams. Unlike the often so-called time-series anomaly detection problem, where the aim is to detect a point $t \in \squareBracket{a, b}$ such that the value of the stream at $t$ is anomalous, here we are interested in the question: Is the stream/time-series as a whole an anomalous object compared with a corpus of normal streams?

Our approach for anomaly detection of streams combines two techniques -- the path signature and the variance norm.
The path signature of a (multi-dimensional) path/time series is a graded, infinite collection of iterated integrals, where the signature of a path $X$ up to degree $N$ is the collection

\begin{equation}\label{eq:sig_intro}
\footnotesize
\begin{aligned}
\Sig^N(\mathbf x) := \bigg( &
\int_{0 < t_1 < \cdots < t_k < 1}  \frac{\mathrm d X_{i_1}}{\mathrm dt}(t_1) \cdot \frac{\mathrm d X_{i_2}}{\mathrm dt}(t_2) \cdots \frac{\mathrm d X_{i_k}}{\mathrm dt}(t_k) \\ & \mathrm dt_1 \cdots \mathrm dt_k 
\bigg)_{\substack{\!\!1 \leq i_1, \ldots, i_k \leq d \\ \!\!k=0, 1, 2, \ldots, N}}.
\end{aligned}
\end{equation}

The signature describes the paths in a global, geometric, and interacting way. For example, the degree $1$ signatures are the increments in different dimensions, and the degree $2$ signatures describe the area formed between the graph of a pair of features and the $45$-degree line. See \cite{MR3727607, chevyrev2016primer, 10041999} for more information about the signature method.
The variance norm is a data-driven metric that coincides with the Mahalanobis distance \cite{mahalanobis2018generalized} when the sample covariance matrix is finite-dimensional and has full column rank. Real data often exhibits rank-deficiency in the sample covariance matrix, e.g. due to multi-colinearity, and while a few techniques have been proposed to use Mahalanobis distance in rank-deficient settings, the variance norm offers a mathematically rigorous approach to handle rank-deficiency, and exposes certain weaknesses in the commonly used variant of Mahalanobis distance. For more details, see Section \ref{sec:maha}. 
Overall, our approach consists of transforming the stream to its path signatures, fitting the data-driven variance norm to the corpus of path signatures, and using the nearest neighbour algorithm as a downstream metric-based anomaly detector.

\paragraph{Summary of contributions}
Our contributions are four-fold. Firstly, to the best of our knowledge, this is the first work proposing using path signatures in an unsupervised learning setting, namely, anomaly detection, emphasising being dimensionless. Our anomaly detector is dimensionless in the sense that it is independent of the unit of measurement of streams. Secondly, we propose the data-driven variance norm, which rigorously generalises the widely used Mahalanobis distance to rank-deficient and also infinitely dimensional settings. Thirdly, using the synergy between path-signature and the variance norm, we are able to show that our anomaly detector is invariant to (1) time-reparametrisation, and (2) concatenations of streams before or after. Moreover, there is a naturally graded discrimination power inherited from the graded structure of the path signature. Finally, we implemented the proposed stream anomaly detector, publicly available on \href{https://github.com/sz85512678/signature_mahalanobis_knn}{\scriptsize{https://github.com/datasig-ac-uk/signature\_mahalanobis\_knn}}. We have empirically compared it with well-known stream anomaly detection methods such as the shapelet method, isolation forest, local outlier factors, with improved performance.

\paragraph{Overview of the paper}
In Section 2, we give a detailed survey of related work on Mahalanobis distances and anomaly detection on streams, and how they relate to this paper. In Section 3, we give careful mathematical treatment of the variance norm and show how it reduces to Mahalanobis distance in certain cases, while exposing certain weaknesses of Mahalanobis distance. In Section 4, we give a more detailed exposure of path signatures and show that our anomaly detector is invariant to both time-reparametrisation and stream concatenations, with a graded discrimination power. In Section 5, we the details of the implementation of our stream anomaly detector, and present numerical results on a range of real-world stream data, showing competitive performances over baselines.

\section{Related work}

\subsection{Mahalanobis distance}
Mahalanobis distance, originally proposed in \cite{mahalanobis2018generalized}, has been widely used for anomaly detections in the literature. 
\cite{kamoi2020mahalanobis, 9455004} considers using the minimal Mahalanobis distance to class means in the intermediate layer of a neural-network classifier for out-of-distribution and adversarial example detection. \cite{5509781} similarly used Mahalanobis distance to the mean in unmanned vehicle data to detect the anomalous operation of the vehicle at single time points. \cite{app12178661} uses Mahalanobis distance as a distance for wind turbine operation anomaly point detection. \cite{PATIL20151054} uses Mahalanobis distance to the mean to detect anomalous operations of insulated gate bipolar transistors. \cite{inproceedings} used Mahalanobis distance to cluster different periods of a time series with a k-means-like algorithm for pattern detection. \cite{gjorgiev2020time} compared using Mahalanobis distance with MSE for reconstruction differences of variational auto-encoders for detecting water system cyber-physical attacks and found using Mahalanobis distance gives a higher overall score (quantified by both accuracy and speed). \cite{pham2017anomaly} used Mahalanobis distance in detecting abnormal users/transactions in the Bitcoin transaction network.
\paragraph{Relationship to our work}
Despite its widespread success, a unified approach to using Mahalanobis distance, however, is lacking. Many authors used Mahalanobis distance to the mean to quantify an anomaly, which we will show suffers from some serious drawbacks. Moreover, it is not clear what one should do when the (empirical) covariance matrix is singular. A variety of approaches have been suggested. 
\cite{brereton2016re} proposed using the sum of squares of the standardised scores of all non-zero principal components to represent the Mahalanobis distance when the sample covariance matrix is rank-deficient; and further suggested a reduced Mahalanobis distance approach where only a certain number of PCs are retained. However the justifications of using the sum of squares of non-zero principal components are not entirely satisfactory, and we show in our approach that this is only correct in certain cases. \cite{PANG2023} used Mahalanobis distance, resolving multi-colinearity with factor analysis to space telemetry series anomaly detection.    \cite{jin2017md} discussed using pseudo-inverse, and feature selection when the covariance matrix is singular; it also discussed using Gamma distribution, Weibull distribution and Box-Cox transformation to set the anomaly threshold.
However, there is no clear theoretical foundation for these techniques.

Our discussion of the variance norm generalises Mahalanobis distance not only to provide a rigorous foundation for handling rank-deficient covariance matrix cases, exposing certain weaknesses of the current approaches, e.g. simply taking the pseudo-inverse; the variance norm is also capable of handling infinite-dimensional feature spaces. Moreover, we firmly note that, in general, it is the nearest-neighbour distance, which we call the conformance distance when used with the variance norm, that should be used in anomaly detection instead of distance to the mean. The distance to the mean approach is often justified by considering the Gaussian mixture model, however, our discussion of the variance norm showed that Mahalanobis distance has a deeper mathematical foundation as a purely data-driven norm, and there is no special reason for treating the mean as a special point. We will give an example of why the distance to the mean can be an unhelpful quantifier for anomalies.

\subsection{Anomaly detections on multivariate time series/streamed data}
To the best of our knowledge, there has not been extensive work studying detecting anomalous time series (i.e. given a corpus of time series, decide whether a given time series is an anomaly or not). \cite{beggel2019time} used shapelet learning on UEA \& UCR time series repository (originally for classfications). Their method is based on using the shapelet feature originally proposed in \cite{ye2009time} and used for time series classification. However, shapelet features are for univariate time series only and generalisation to interacting multi-variate time series is unclear. Nevertheless, we compare our approach to this method in one of our numerical experiments, restricting to univariate streams. 
\cite{hyndman2015large} proposes reducing the dimensionality by extracting some representative statistical features from each time series (e.g., mean and first order of autocorrelation) and then applying PCA. Outlier time series are detected by their deviation from the highest density region in the PCA space, which is defined by the first two principal components.
\cite{10.1007/s10115-017-1067-8} uses dynamic time wrapping to define a similarity function between each pair of time series. Once a similarity function is defined, the time series can be clustered in different groups using the similarity function. The outlier score is then computed for each time series based on its distance to its closest centroid. In our numerical experiments, we compare SigMahaKNN with well-known anomaly detection methods: isolation forest and local outlier factor, using signature or non-signature-based feature extraction techniques for streams.

Several related fields, such as time-series classifications, also rely on feature extractions from time-series as a key component. \cite{li2021dynamic} considers graph embedding time series using SDNE (structured deep network embedding); \cite{10.1007/978-3-030-30490-4_56} uses GAN based on LSTMRNN for both the generator and discriminators for multivariate time-series.  \cite{NEURIPS2022_194b8dac} used self-supervised learning to construct representations for time series classifications. \cite{morrill2020generalised} considers the signature features for time-series classifications. \cite{middlehurst2023bake} recently benchmarked performances of multi-variate time series classifications algorithms and found Hydra+MultiROCKET \cite{dempster2023hydra} and HIVE-COTEv2 \cite{middlehurst2021hive}, perform significantly better than other approaches on both the current and new TSC problems. These representations for time-series classifications could arguably be tried in our unsupervised anomaly detection context. However, in this paper, we chose to focus on the signature features with Mahalanobis distance, partly due to its strong theoretical properties such as concatenation invariance and reparametrisation invariance.

\section{Variance norm, a data-driven metric} \label{sec:maha}
In this section, we give a careful discussion of the variance norm. The data $x_1, \dots, x_n$ are always assumed to be mean-centred. New data such as $v$ is centred around the mean of $x_1, \dots, x_n$.

\subsection{Variance norm as a data-driven quadratic form}
\begin{definition}
  Given data $x_i, i=0, 1, \dots, n-1$, $x_i \in V$, a possibly infinite dimensional vector space, let $p(v, v) = \sup_{q(u,u)\leq 1} u(v)^2$, where for $u \in U:= V^*$, the continuous dual space of $V$ consisting of continuous linear transformations $V \to \R$, and we define $q(u,u) = \sum_{i=1}^n u^2(x_i)$. We say $\sqrt{p(v,v)}$ is the variance norm of $v$.
\end{definition}

In order to understand the quadratic form $p(.,.)$ and $q(., .)$, we will consider the special case where $V$ is a finite-dimensional vector space and we show that in this case, the quadratic form $p$ is equivalent to the Mahalanobis distance when the data has no multi-colinearity. Our generic, mathematical definition allows us to handle cases where $V$ is infinite-dimensional as well as when there is multi-colinearity of $x_i$ so that the covariance matrix is not invertible. The following theorem is our main result on the variance norm.

\begin{theorem}
    Let $V$ be a finite-dimensional (real) vector space. Given $x_1, x_2, \dots, x_n \in V$, $v \in V$, we have
    \begin{enumerate}
        \item If $v \in \Span{x_i}$, then $p(v,v) = v^T \bracket{X^T X}^\dagger v$, where $\dagger$ denotes Moore-Penrose peusdo-inverse, and $X \in \R^{n\times d}$ whose rows are $x_i$ with respect to some basis of $V$. 
        \item If $v \notin \Span{x_i}$, then $p(v,v) = \infty$.
    \end{enumerate}
\end{theorem}

\begin{proof}
Let $U_0 \subseteq U, U_0 = \bracket{\Span{x_i}}^o$, the annihilator space of the linear span of the data, with a basis $\set{u_i^0}$, and extend it to a basis of the whole space $U$ so that $U = U_0 + U_1$ with $\set{u_j^1}$ being a $q$-orthonormal basis of $U_1$, that is, we have $q(u_j^1, u_j^1) = 1, q(u_i^1, u_j^1) = 0$, where we define $q(u_1, u_2) = \sum_{i=1}^n u_1(x_i) u_2(x_i)$. Thus $\set{u_i^0, u_j^1}$ is a $q$-orthogonal basis of $U$, let its dual basis in $V$ be ${v_i^0, v_j^1}$. Naturally, any vector $v \in V$ can be written as $v = \sum \mu_i v_i^1 + v^0$, where $v^0 \in \Span{v_i^0}$. We make the following claims:

\begin{enumerate}
    \item $v \in \Span{x_i}$ iff $v^0 = 0$.
    \item If $v \in \Span{x_i}$, then $p(v,v) = \sum \mu_i^2$.
    \item If $v \notin \Span{x_i}$, then $p(v,v) = \infty$.
\end{enumerate}

Consider the first statement, if $v \in \Span{x_i}$ and $v^0 \neq 0$, then $u_i^0(v) = 0$ for any $i$, but because $v^0 \in \Span{v_i^0}$, there exists $u^0$ such that $u^0(v^0) = 1$, at the same time $u^0(v) = 0$ and $u^0(v_i^1)=0$, then $u(v) = u\bracket{\sum \mu_i v_i^1 + v^0}$ gives a contradiction. Therefore $v \in \Span{x_i}$ implies $v^0=0$. Conversely, if $v^0 = 0$, then we have $u_i^0(v) = 0$ for any $i$, hence $v \in \bracket{\Span{x_i}}^{oo}$, the annihilator of annihilator. As we work in the finite dimensional space, we have $\bracket{\Span{x_i}}^{oo} = \Span{x_i}$. The first statement then follows.

For the second and the third statements, we note that $q(a u_1 + b u_2, a u_1 + b u_2) = a^2 q(u_1, u_1) + 
b^2 q(u_2, u_2) + 2 a b q(u_1, u_2)$. Therefore for $u = \sum a_i u_i^1 + u^0$, where $u^0 \in U_0$, we have $q(u,u) = \sum a_i^2$ because $q(u, u^0) = 0$. In particular, $u^0$ can be arbitrary under the constraint $q(u,u) \leq 1$. Next, expanding $p(v,v)$ with $v = \sum \mu_i v_i^1 + v^0$, we have that $p(v,v) = \sup_{q(u,u) \leq 1}  \sum \mu_i \mu_j u(v_i^1) u(v_j^1) + 2 \sum_i \mu_i u(v_i^1) u(v^0)
+ u(v^0)^2$. If $v \in \Span{x_i}$, then necessarily $v^0 = 0$ and it follows that $p(v,v) = \sup_{\sum a_i^2 \leq 1} \sum \mu_i \mu_j a_i a_j  = \bracket{\sum \mu_i^2}^2$ by Cauchy-Schwarz inequality. On the other hand, if $v \notin \Span{x_i}$ so that $v^0 \neq 0$, we may set $a_i = 0$ and obtain an unconstrained expression $\sup_{u\in U_0} (u(v^0))^2$, which gives infinity.

In order to progress further, we need to fix a basis. Let us suppose our data vectors $x_i \in V$ come in a measurement basis $B_0$ of the vector space $V$. We note that our definition of $p(., .)$ has not involved a basis of $V$, thus the quantities derived are independent of the measurement basis. Under the measurement basis and the Euclidean inner product induced by this basis, we may assemble $x_i \in V$ as rows of $X \in \R^{n\times d}$, and compute the singular value decomposition $X = WDY^T$, where $D_{ii} = 0$ for $i > r$. 

Let $B_1$ be the basis for $V$ under change of basis matrix $Y$ from $B_0$, that is, suppose $B_0 = \set{b_1^0, b_2^0, \dots, b_d^0}$ and $B_1 = \set{b_1^2, b_2^1, \dots, b_d^2}$ then we have 
\begin{align*}
    b_1^0 = y_{11}b_1^1 + y_{21}b_2^1 + y_{31}b_3^1 + \dots, \\
    b_2^0 = y_{12}b_1^1 + y_{22}b_2^2 + y_{32}b_3^1 + \dots, \\
    \dots.
\end{align*}
Under $B_1$ basis, the data matrix $X$ becomes $XY$, which equals to $WD$. 
Let $B_1^*$ be the dual basis for $U = V^*$. Let $\bar{B_1^*} = \set{b_{11}^*, b_{12}^*, \dots, b_{1r}^*}$ be the set of the first $r$ vectors in $B_1^*$, and $\hat{B_1^*}$ be the last $d-r$ vectors of $B_1^*$. Because that in $B_1$ basis, the data matrix $X$ becomes $WD$, with the last $d-r$ columns of the diagonal matrix $D$ being zero, we see that $\hat{B_1^*}$ is a basis for $U_0$, the annihilator of the span of rows of $X$. Therefore $\bar{B_1^*}$ forms a basis of $U_1$, the space that, when added to $U_0$, gives us the whole space $U$. Furthermore, if $x_k = \sum a_l^k b_l^1$, we can compute
\begin{equation*}
\begin{aligned}
    q(b_{1i}^*, b_{1j}^*) &= \sum_{k=1}^n b_{1i}^*(x_k) b_{1j}^*(x_k) 
    = \sum_{k} a_i^k a_j^k \\ &= \sum_{k} (XY)_{ki} (XY)_{kj} = \sum_{k} \left(D^T W^T WD\right)_{ij} \\ &= D_{ii}^2 \delta_{ij}. 
\end{aligned}
\end{equation*}

Thus we see that $\bar{B_1^*}$ is an orthogonal basis for $U_1$, and $\bar{B_2^*} = \set{D_{ii}^{-1} b_{1i}^*}$ is an $q$-orthonormal basis for $U_1$. An arbitrary (row) vector $v \in V$, if $v \in \Span{x_i}$, given in the measurement basis becomes $v \bar{Y} \bar{D}^{-1}$ in the dual basis of $\bar{B_2^*}$, where $\bar{Y}$ is the first $r$ columns of $Y$ and thus $p(v,v) = v \bar{Y} \bar{D}^{-1} \bracket{v \bar{Y} \bar{D}^{-1}}^T  = v \bar{Y} \bar{D}^{-2} \bar{Y}^T v^T$; if we view $v$ as a column vector, we arrive at $p(v,v) = v^T \bracket{\bar{Y} \bar{D}^{-2} \bar{Y}^T} v = v^T \bracket{X^T X}^{\dagger} v$. If v is not in $\Span{x_i}$, as discussed before, we have $p(v,v) = \infty$.
\end{proof}

The variance norm will be used as part of the conformance distance for anomaly detection. 

\begin{definition}
    Let $V$ be a finite dimensional vector space, $\mathcal{C} = \set{x_1, \dots, x_n} \subseteq V$, $y \in V$. The conformance distance is $d(y, \mathcal{C}) = \min_{x_i \in \mathcal{C}} p(y-x_i, y-x_i)$.
\end{definition}

\subsection{Discussion}

\subsubsection{Comparison to Mahalanobis distance}

The Mahalanobis distance, commonly defined between a point $y$ and a distribution $D$ is 
\begin{equation}
    d_M(y, D) = \sqrt{\bracket{y-\mu}^T S^{-1} \bracket{y-\mu}},
\end{equation}
where $\mu$ is the mean of $D$ and $S$ is the covariance matrix of $D$, assumed to be positive-definite. Replacing $D$ with the empirical distribution defined by the corpus, we have
\begin{equation}
    d_M(y, X) = \sqrt{\bracket{y-\mu}^T \bracket{(X-\mu)^T(X-\mu)}^{-1} \bracket{y-\mu}}, \label{eqn:3}
\end{equation}
where $\mu = \frac{1}{n} \sum_i x_i$.
\eqref{eqn:3} has two issues:
\begin{enumerate}
    \item Computing a form of distance to the mean for anomaly detection may not be suitable for certain situations. For example, suppose that the corpus consists of points on the unit circle, centred at the origin, where the points are uniformly distributed with an angle $\theta \in \squareBracket{0, \pi}$. Suppose the point of interest $y$ lies on the circle but with an angle $\frac{3}{2}\pi$. $y$ is clearly an exceptionality but it has the same distance to the mean as any point in the corpus. 

    \item It is not clear how one should proceed when the empirical covariance matrix $(X-\mu)^T (X-\mu)$ is rank deficient/singular. One might think of replacing the inverse with the pseudo-inverse, which coincides with a special case of our conformance distance.
\end{enumerate}

In our approach, we have $d(y, \mathcal{C})$ equals to $\min_{x_i \in \mathcal{C}} p(y-x_i, y-x_i) 
= (y-x_i)^T \bracket{X^T X}^\dagger (y-x_i)$, or $\infty$. In the special case where $X$ has full column rank and $(y-x_i)$ belongs to the column space of $X$, conformance distance is equal to nearest neighbour Mahalanobis distance. In general, they differ. We have already seen why it makes sense to use the nearest neighbour distance, and why the inverse needs to be modified in some cases, but the most striking feature of our conformance distance is that it equals infinity in some cases, as we see in Section \ref{section:whyMahainf} it is a key and beneficial feature.

\subsubsection{Why the variance norm should be infinity in some cases}\label{section:whyMahainf}
Taking the pseudo-inverse computationally makes $\tilde{d}(y, X) := \min_{x_i \in X} \bracket{y-x_i}^T (X^TX)^{\dagger} (y-x_i)$ insensitive to certain modifications of $y$. To see this, note that by the definition of the pseudo-inverse, if the SVD of $X$ is $X = WDY^T$, then
$(X^TX)^\dagger = V (D^T D)^{\dagger} V^T $, where $(D^T D)^{\dagger}$ is diagonal with the first $r$ elements on the diagonal being the inverse of squares of singular values of $X$, and other diagonal elements being $0$. Let $z$ be in the span of the last $d-r$ columns of $V^T$. Now it is clear that we have
\begin{equation}
    \tilde{d}(y, X) = \tilde{d}(y + z, X),
\end{equation}
as $zV$ has only non-zeros in the last $d-r$ entries. 

From an anomaly detection point of view, the above results in, e.g. $0 = \tilde{d}(X_i, X) = \tilde{d}(X_i + z, X)$, however as $X_i + z$ is outside the row space of $X$, it might plausibly be classified as an anomaly, in fact, a different type of anomaly. This coincides with our quadratic form $p$ and is the method we propose.

\section{Signatures for unsupervised anomaly detection on streams}\label{sec:sig}

Building on the variance norm and conformance distance introduced in the last section, here we formally define a stream as a mathematical object and study anomalous streams with path signatures and conformance distance. The combination of the two yields many theoretical properties: stream reparametrisation invariance, stream concatenation invariance and graded discrimination power.

\subsection{Streams of data and signature features}

Below we give a formal definition of a stream of data.

\begin{definition}[Stream of data]
The space of streams of data in a set $\mathcal X$ is defined as
$$\mathcal S(\mathcal X) := \{ \mathbf x = (x_1, \ldots, x_n) \,:\, x_i\in \mathcal X, n\in \mathbb N\}.$$
\end{definition}

\begin{example}
When a person writes a character by hand, the stroke of the pen naturally determines a path. If we record the trajectory we obtain a two-dimensional stream of data $\mathbf x = \left ( (x_1, y_1), (x_2, y_2), \ldots, (x_n, y_n) \right ) \in \mathcal S(\mathbb R^2)$. If we record the stroke of a different writer, the associated stream of data could have a different number of points. The distance between successive points may also vary.
\end{example}

We give a formal definition of path signatures.

\begin{definition}[Signature]
Let $\mathbf x = (x_1, \ldots, x_n) \in \mathcal S(\mathbb R^d)$ be a stream of data in $d$ dimensions. Let $X=(X_1, \ldots, X_d):[0, 1]\to \mathbb R^d$ be such that $X\left ( \frac{i}{n - 1} \right )=x_{i+1}$ for $i=0, 1, \ldots, n-1$ and linear interpolation in between. Then, we define the signature of $\mathbf x$ of order $N\in \mathbb N$ as
\begin{equation}\label{eq:sig}
\footnotesize
\begin{aligned}
\Sig^N(\mathbf x) := \bigg( &
\int_{0 < t_1 < \cdots < t_k < 1}  \frac{\mathrm d X_{i_1}}{\mathrm dt}(t_1) \cdot \frac{\mathrm d X_{i_2}}{\mathrm dt}(t_2) \cdots \frac{\mathrm d X_{i_k}}{\mathrm dt}(t_k) \\ & \mathrm dt_1 \cdots \mathrm dt_k 
\bigg)_{\substack{\!\!1 \leq i_1, \ldots, i_k \leq d \\ \!\!k=0, 1, 2, \ldots, N}}.
\end{aligned}
\end{equation}

\end{definition}

The signature of a stream of data is a vector of scalars. The dimension of this vector is
$$d_N:=1+d+d^2 + \cdots + d^N=\frac{d^{N+1}-1}{d-1}.$$ The signatures are unique and have universal non-linearity.
\paragraph{Uniqueness}:
Hambly and Lyons \cite{hambly2010uniqueness}show that under mild assumptions, the full collection of features Sig(x) uniquely determines x up to translations and reparametrisations.

\paragraph{Universal non-linearity}:
Linear functionals on the signature are dense in the set of functions on x. Suppose we wish to learn the function f that maps data x to labels y, the universal non-linearity property states that, under some assumptions, for any$ \epsilon > 0$, there exists a linear function L such that
$\| f(x) - L(Sig(x)) \| \leq \epsilon
$. See \cite{lyons2007differential}.

\subsection{SigMahaKNN, detecting anomalies in streamed data} \label{sec:var}
Finally, we define anomalous streams using path signatures and conformance distance.
Let $\mathcal C\subset \mathcal S(\mathbb R^d)$ be a finite corpus (or empirical measure) of streams of data. Let $\Sig^N$ be the signature of order $N\in \mathbb N$. Then $\lVert \cdot \rVert_{\Sig^N(\mathcal C)}$ is the conformance distance associated with the empirical measure of $\Sig^N(\mathcal C)$. We let $\sigNorm{\cdot}$ be our anomaly score, that is, streams with conformance distance higher than a user-determined threshold \footnote{This threshold could be naturally calibrated according to the desired level of discriminating power, see the next section.} are anomalies. We call our anomaly detection algorithm SigMahaKNN, a combination of the signature, a generalisation of Mahalanobis distance, and k-nearest neighbour.

\subsection{Properties of SigMahaKNN}

\subsubsection{Reparametrisations invariance}
A reparametrisation of a path $X: \squareBracket{a, b} \to \R$ is a path $\tilde{X}: \squareBracket{a, b} \to \R$ where $\tilde{X}_t = X_{\psi_t}$ where $\psi$ is a surjective, continuous, non-decreasing function $\psi: \squareBracket{a, b} \to \squareBracket{a, b}$. The paths signature is invariant under reparametrisation \cite{chevyrev2016primer}. Therefore SigMahaKNN computes an anomaly score that is invariant to path reparametrisations. We give a concrete example of reparametrisation. Consider writing a digit on a sheet of paper. The handwriting forms a path $\squareBracket{0,1} \to \R^2$. A reparametrisation corresponds to a change in speed of writing while keeping the shape of the digit written exactly the same. This transformation should not change whether the written shape is an anomaly, as reflected by the reparametrisation invariance of SigMahaKNN.
\begin{theorem}
    Let $\tilde{X}_t$ be a reparametrisation of $X$. Then we have $\sigNorm{\tilde{X}} = \sigNorm{X}$.
\end{theorem}

\subsubsection{Concatenation invariance}

Let $X: \squareBracket{a, b} \to \R^d$ and $Y: \squareBracket{b, c} \to \R^d$ be two paths. We define their concatenation as the path $X*Y: \squareBracket{a, c} \to \R^d$ for which $\bracket{X*Y}_t = X_t$ for $t \in \squareBracket{a, b}$ and $\bracket{X*Y}_t = X_b + (Y_t - Y_b)$ for $t \in \squareBracket{b, c}$ \cite{chevyrev2016primer}. We then have (Chen's identity)
\begin{eqnarray}
S(X * Y)_{a,c} = S(X)_{a,b} \otimes S(Y)_{b,c},
\end{eqnarray}
where we have represented the signature by a formal power series
\begin{eqnarray}
S(X)_{a,b} = \sum_{k = 0}^\infty  \sum_{i_1,\ldots, i_k \in \{1,\ldots d\}} S(X)^{i_1,\ldots, i_k}_{a,b} e_{i_1}\ldots e_{i_k},
\end{eqnarray}
The tensor product is defined as 
\begin{eqnarray}
e_{i_1}\ldots e_{i_k} \otimes e_{j_1}\ldots e_{j_m} = e_{i_1}\ldots e_{i_k}e_{j_1}\ldots e_{j_m}.
\end{eqnarray}
The product $\otimes$ then extends uniquely and linearly to all power series. We demonstrate the first few terms of the product in the following expression

\begin{align*}
\left(\sum_{k = 0}^\infty \sum_{i_1,\ldots, i_k \in \{1,\ldots d\}} \lambda_{i_1,\ldots,i_k}e_{i_1}\ldots e_{i_k}\right) & \\
\otimes \left(\sum_{k = 0}^\infty \sum_{i_1,\ldots, i_k \in \{1,\ldots d\}} \mu_{i_1,\ldots,i_k}e_{i_1}\ldots e_{i_k} \right) & \\
= \lambda_0\mu_0 + \sum_{i=1}^d (\lambda_0\mu_i + \lambda_i\mu_0)e_i & \\
+ \sum_{i,j=1}^d \left( \lambda_0 \mu_{i,j} + \lambda_i\mu_j + \lambda_{i,j}\mu_0\right)e_ie_j + \ldots. &
\end{align*}


We see that concatenation with $X$ results in a linear transformation of the signatures of $Y$. Because a linear transformation is equivalent to a change of basis, and the variance norm is independent of basis (it is defined without using any basis), it follows that the variance norm with path signatures is invariant under concatenations by paths on the left. Similarly, the variance norm with path signatures is invariant under concatenation by paths on the right. As the conformance distance is computed by taking the minimal variance norms, it is also invariant under concatenations of paths. That is
\begin{theorem}
    Let $\mathcal{C}$ be a corpus of paths and $X, Y$ be paths. Define $\mathcal{C} + X$ be the set of paths consisting of $Z*X$ for $Z \in \mathcal{C}$. Then we have that $\lVert Y \rVert_{\Sig^N(\mathcal C)}
    = \lVert Y*X \rVert_{\Sig^N(\mathcal{C}+X)}$.
\end{theorem}

\subsubsection{Naturally graded discriminating powers}
The graded structure of the signature given by its order gives naturally graded discriminating powers. In particular, the variance norm with signature features is monotonically increasing with respect to the signature order and will reach infinity for a finite order of the signature.

\begin{proposition}
 Let $\mathcal C \subset \mathcal S(\mathbb R^d)$ be a finite corpus. Let $W$ be a path. Then, $\sigNorm{W}$ is non-decreasing as a function of $N$. 
\end{proposition}

\begin{proof} 

Let $M\geq N$. Let $y_N, y_M$ be the $M, N$-order of signatures for the stream $W$, $X_N, X_M$ be the $M, N$-order of signatures of the corpus (represented in some basis, with the basis of for $X_N$ being a subset of the basis for $X_M$) respectively. Note that the first $N$ entries/entry for each row of $y_M, X_M$ are identical to $y_N, X_N$ respectively.

We discuss two cases. If $y_N$ is not in the span of the rows of $X_N$, then $y_M$ is not in the span of the rows of $X_M$, hence in both conformance distances are infinity. 

In case that $y_N$ is in the span of the rows of $X_N$,
if $y_N, y_M$ are both in the span of the rows of $X_N, X_M$ respectively, then the variance norm using signature of level $M$ equals to the variance norm using signature of level $N$ plus non-negative terms. This is because of the formula $p(v, v) = v^T \bracket{X^T X}^\dagger v$.
If $y_N$ is in the span of $X_N$ but $y_M$ is not in the span of $X_M$, then in the latter case the variance norm is infinity, so clearly it is increasing. 
\end{proof}

Moreover, for a sufficiently high resolution, any stream of data not belonging to the corpus has infinite variance:

\begin{proposition}
 Let $\mathcal C \subset \mathcal S(\mathbb R^d)$ be a finite corpus. Let $\mathbf Y \in \mathcal S(\mathbb R^d)$ be a stream of data that does not belong to the corpus, $\mathbf Y \not\in \mathcal C$. Then, there exists $N$ large enough such that
$$\lVert Y \rVert_{\Sig^n(\mathcal C)} = \infty \quad \forall\,n\geq N.$$
\end{proposition}

\begin{proof}
If $\mathbf Y \notin \mathcal C$, there exists $N$ large enough such that $\Sig^N(\mathbf y)$ is linearly independent to $\Sig^N(\mathcal C)$ \cite{lyons2007differential}. Therefore, the variance norm and hence the conformance distance are infinity. Going beyond $N$, $\Sig^N(\mathbf y)$ is still linearly independent to $\Sig^N(\mathcal C)$, as the first $d^N$ entries are the same as the level $N$ vectors.

\end{proof}

\subsubsection{A dimensionless anomaly detector}
Changing the measurement units of the stream channels induces a linear transformation on the signatures of the streams using the formula in \eqref{eq:sig}. \footnote{To be clear, the higher order signatures will be composed of multiple channel linear transformations, but the composition of linear transformations is still linear.} Because the variance norm is independent of linear transformations, we have
\begin{theorem}
    Let $\mathcal{C}$ be a corpus of paths and $Y$ be a path. Suppose each path in $\mathcal{C}$ has $d$ dimensions. Let $A \in \R^{n \times d}$ non-singular, and define $AX$ be a path where $A$ is applied to each point in the path. Then we have $\lVert Y \rVert_{\Sig^N(\mathcal C)}
    = \lVert AY \rVert_{\Sig^N(A\mathcal{C})}$.
\end{theorem}

\section{Software implementation and Numerical experiments}

The software for anomaly detection on streamed data, and code for the reproduction of all our experiments is available at  
\href{https://github.com/sz85512678/signature_mahalanobis_knn}{{https://github.com/datasig-ac-uk/signature\_mahalanobis\_knn}}.
\subsection{Software contribution}

\subsubsection{Implementation of SigMahaKNN}
\begin{alg}[SigMahaKNN]\label{alg1}
Given a corpus of $n$ streams, train a function mapping a stream to an anomaly score in $\squareBracket{0, \infty}$. Parameters: (1) path-signature related: signature augmentations, scaling, signature-windowing, truncation depths (2) variance norm related: subspace threshold, svd threshold.

\begin{enumerate}[topsep=0pt,itemsep=-1ex,partopsep=1ex,parsep=1ex]
	\item Compute signatures of the streams, forming a corpus matrix $X$. Centre the rows of $X$ on its mean, $\mu$.

	\item Compute a truncated SVD factorisation of $X$ of the form, $X= U \Sigma V^T$, where $U \in \R^{n \times k}, \Sigma \in \R^{k \times k}, V^T
 \in \R^{k \times m}$, where $k$ is the numerical rank of $X$ such that the $k$-th largest singular value of $X$ is the smallest singular value bigger or equal to the svd threshold.

    \item Compute the signature of the input stream, and centre it by subtracting $\mu$ from it, denote the result by $Y$. For each pair $\bracket{X_i, Y}$, where
    $X_i$ is the $i$-th row of $X$, compute the variance norm of $z:= X_i -Y$ by ($z$ viewed as a column vector): 
        \begin{enumerate}
            \item If $\normTwo{z} < 10^{-15}$, return $0$.
            \item If $\frac{\normTwo {VV^Tz -z}}{\normTwo{z}} > \text{subspace threshold}$, return $\infty$.
            \item Otherwise, return $z^T V \Sigma^{-2} V^T z$.
        \end{enumerate}

    \item Return $\min_i \| Y-X_i\|_X$ by using the variance norm metric and the sklearn nearest neighbour.
\end{enumerate}
\end{alg}

\subsubsection{Discussion}
\paragraph{Threshold parameters in the variance norm}
In practice, $X$ is often full algebraic rank due to rounding errors but numerically rank-deficient, that is, its trailing singular values are near zero. We set a threshold to compute an approximate SVD of $X$, so as to handle numerical rank-deficiency. Due to this approximation, the basis independent property of the variance norm will only hold approximately in practice. Moreover, due to numerical errors, we need a subspace threshold to determine if a given vector belongs to the span of rows of $X$.

\paragraph{Setting a threshold given anomaly scores}
The nearest neighbour distance gives us a number that a user often needs to compare against a threshold. One way to set the threshold in practice is to use the following procedure. We split the corpus into two equal-sized parts and compute the empirical cdf of the nearest neighbour distance for one part using the other part as the corpus. One can then set a threshold distance by choosing an appropriate tail quantile in the empirical cdf.\

\paragraph{Signature transformations}
The practical performance of signatures is often improved by considering various stream augmentations \cite{morrill2020generalised}. We will use the time-augmentation, invisibility-reset and lead-lag transformations in some of our numerical experiments. For more details of these transformations see \cite{morrill2020generalised}.

\subsection{Experimental results}
\subsubsection{Dataset and methods compared}
We test our approach on four datasets: Handwritten digits, marine vessel traffic data, language data and a selection of univariate time series from the UCR repository. 
For the first three multivariate streams, we compare our methods with isolation forest\cite{4781136} and local outlier factor method\cite{10.1145/342009.335388}. These two methods are well-known anomaly detection methods, and in order to use them on streams, we use them with either the moment features (mean and covariance of different dimensions of the stream), or the signature features. We thus have four baselines: IF-M (Isolation Forest with Moment features), IF-S (Isolation Forest with Signature features), LOF-M (Local Outlier Factor with Moment features), LOF-S (Local Outlier Factor with Signature features). For univariate time series we compare our method with a specialised univariate time series anomaly detection technique based on shapelet by \cite{beggel2019time}.
We mostly report on AUC as a measure of accuracy for anomaly detection. The full experimental results including run-time are included in the supplementary materials.

\subsubsection{Handwritten digits}
We evaluate our proposed method using the PenDigits-orig data set \cite{Dua:2019}. This data set consists of 10\,992 instances of hand-written digits captured from 44 subjects using a digital tablet and stylus, with each digit represented approximately equally frequently. Each instance is represented as a 2-dimensional stream, based on sampling the stylus position at 10Hz.

We apply the PenDigits data to unsupervised anomaly detection by defining $\mathcal{I}_{\text{normal}}$ as the set of instances representing digit $m$. We define $\mathcal{C}$ as the subset of $\mathcal{I}_{\text{normal}}$ labelled as `training' by the annotators. Furthermore, we define $\mathcal{Y}$ as the set of instances labelled as `testing' by the annotators ($|\mathcal{Y}| = 3498$). Finally, we define $\mathcal{I}_{\text{anomaly}}$ as the subset of $\mathcal{Y}$ not representing digit $m$. Considering all possible digits, we obtain on average $|\mathcal{C}| = 749.4$, $|\mathcal{I}_{\text{anomaly}}| = 3\,148.2$. Assuming that the digit class is invariant to translation and scaling, we apply Min-Max normalisation to each individual stream.

Table~\ref{tab:digit-performance} displays results based on taking signatures of order $N \in [1..5]$ and without any stream transformations applied. The results are based on aggregating conformance values across the set of possible digits before computing the ROC AUC. As we observe, performance increases monotonically from 0.901 ($N=1$) to 0.989 ($N=5$). 


\begin{table}[h!]
\centering
\scriptsize 
\setlength{\tabcolsep}{3pt} 
\begin{tabular}{|l|l|l|l|l|l|}
\hline
        & N=1    & N=2    & N=3    & N=4    & N=5    \\ \hline
SigMahaKNN & 0.870 & 0.942 & \textbf{0.948} & \textbf{0.954} & \textbf{0.956} \\ \hline
IF-M     & -      & -      & -      & -      & 0.618 \\ \hline
IF-S     & \textbf{0.888} & 0.931 & 0.916 & 0.875 & 0.834 \\ \hline
LOF-M    & -      & -      & -      & -      & 0.514 \\ \hline
LOF-S    & 0.563 & 0.584 & 0.582 & 0.582 & 0.582 \\ \hline
\end{tabular}
\caption{Handwritten digits data: performance quantified using ROC AUC in response to signature order $N$. 
Bootstrapped standard errors based on $10^4$ samples are around $0.003$.
}
\label{tab:digit-performance}
\end{table}

\subsubsection{Marine vessel traffic data}
Next, we consider a sample of marine vessel traffic data\footnote{\url{https://coast.noaa.gov/htdata/CMSP/AISDataHandler/2017/AIS_2017_01_Zone17.zip}, accessed May 2020.}, based on the automatic identification system (AIS) which reports a ship's geographical position alongside other vessel information. The AIS data that we consider were collected by the US Coast Guard in January 2017, with a total of 31\,884\,021 geographical positions recorded for 6\,282 distinct vessel identifiers. We consider the stream of timestamped latitude/longitude position data associated with each vessel a representation of the vessel's path. 

We prepare the marine vessel data by retaining only those data points with a valid associated vessel identifier. In addition, we discard vessels with any missing or invalid vessel length information. Next, to help constrain computation time, we compress each stream by retaining a given position only if its distance relative to the previously retained position exceeds a threshold of 10m. Finally, to help ensure that streams are faithful representations of ship movement, we retain only those vessels whose distance between initial and final positions exceeds 5km. To evaluate the effect of stream length on performance, we disintegrate streams so that the length $D$ between initial and final points in each sub-stream remains constant with $D \in \{4\text{km}, 8\text{km}, 16\text{km}, 32\text{km}\}$. After disintegrating streams, we retain only those sub-streams whose maximum distance between successive points is less than 1km. 

We partition the data by deeming a sub-stream normal if it belongs to a vessel with a reported vessel length greater than 100m. Conversely, we deem sub-steams anomalous if they belong to vessels with a reported length less than or equal to 50m. We obtain the corpus $\mathcal{C}$ from 607 vessels, whose sub-streams total between 10\,111 ($D=32$km) and 104\,369 ($D=4$km); we obtain the subset of normal instances used for testing $\mathcal{I}_{\text{normal}} \setminus \mathcal{C}$ from 607 vessels, whose sub-streams total between 11\,254 ($D=32$km) and 114\,071 ($D=4$km); lastly we obtain the set of anomalous instances $\mathcal{I}_{\text{anomaly}}$ from 997 vessels whose sub-streams total between 8\,890 ($D=32$km) and 123\,237 ($D=4$km). To account for any imbalance in the number of sub-streams associated with vessels, we use for each of the aforementioned three subsets a weighted sample of 5\,000 instances.

After computing sub-streams and transforming them as described, we apply Min-Max normalisation with respect to the corpus $\mathcal{C}$. To account for velocity, we incorporate the difference between successive timestamps as an additional dimension (time augmentation). 

We report results based on taking signatures of order $N=3$. For comparison, as a baseline approach, we summarise each sub-stream by estimating its component-wise mean and covariance, retaining the upper triangular part of the covariance matrix. This results in feature vectors of dimensionality $\frac{1}{2}(n^2 + 3n)$ which we provide as the input to an isolation forest \cite{4781136}. We train the isolation forest using 100 trees and for each tree in the ensemble using 256 samples represented by a single random feature.

Table~\ref{tab:marine-path-performance} to Table \ref{tab:marine-path-performance-updated-LOF-S} display results for our proposed approach in comparison to the baselines, for combinations of stream transformations and values of the sub-stream length $D$. Signature conformance yields higher ROC AUC scores than the baseline for the majority of parameter combinations. The maximum ROC AUC score of 0.891 is for a combination of lead-lag, time differences, and invisibility reset transformations with $D=32$km, using the signature conformance. Compared to the best-performing baseline parameter combination for $D=32$km (IF-S), this represents a performance gain of $10$ percentage points.

\begin{table}[h]
\centering
\scriptsize 
\setlength{\tabcolsep}{3pt} 
\begin{tabular}{ccc|cccc}
 \multicolumn{3}{c|}{\multirow{2}{*}{Transformation}} &
 \multicolumn{4}{c}{Conformance $\dist(\;\cdot \; ; \Sig^3(\mathcal C))$} \\
 \cline{4-7}
 & & & \multicolumn{4}{c}{Sub-stream length $D$} \\
 Lead-lag & Time-Diff & Inv.~Reset &
 $4$km & $8$km & $16$km & $32$km \\
 \hline
No & No & No    & 0.723 & 0.706 & 0.705 & 0.740 \\
No & No & Yes   & 0.776 & 0.789 & 0.785 & 0.805 \\
No & Yes & No   & 0.810 & 0.813 & 0.818 & 0.848 \\
No & Yes & Yes  & 0.839 & 0.860 & 0.863 & 0.879 \\
Yes & No & No   & 0.811 & 0.835 & 0.824 & 0.837 \\
Yes & No & Yes  & 0.812 & 0.835 & 0.833 & 0.855 \\
Yes & Yes & No  & 0.845 & 0.861 & 0.862 & 0.877 \\
Yes & Yes & Yes & \textbf{0.848} & \textbf{0.863} & \textbf{0.870} & \textbf{0.891}
\end{tabular}
\caption{SigMahaKNN on Marine vessel traffic data: performance quantified using ROC AUC for combinations of stream transformations and sub-stream length $D$. Best across all transformations are in bold.}
\label{tab:marine-path-performance}
\end{table}

\begin{table}[h]
\centering
\scriptsize 
\setlength{\tabcolsep}{3pt} 
\begin{tabular}{ccc|cccc}
 \multicolumn{3}{c|}{\multirow{2}{*}{Transformation}} &
 \multicolumn{4}{c}{Conformance $\dist(\;\cdot \; ; \Sig^3(\mathcal C))$} \\
 \cline{4-7}
 & & & \multicolumn{4}{c}{Sub-stream length $D$} \\
 Lead-lag & Time-Diff & Inv.~Reset &
 $4$km & $8$km & $16$km & $32$km \\
 \hline
No & No & No    & 0.714 & 0.712 & 0.727 & 0.727 \\
No & No & Yes   & 0.781 & \textbf{0.785} & 0.776 & \textbf{0.790} \\
No & Yes & No   & 0.767 & 0.772 & 0.786 & 0.804 \\
No & Yes & Yes  & 0.830 & 0.823 & \textbf{0.831} & 0.828 \\
Yes & No & No   & 0.696 & 0.704 & 0.711 & 0.724 \\
Yes & No & Yes  & 0.758 & 0.759 & 0.767 & 0.773 \\
Yes & Yes & No  & 0.747 & 0.763 & 0.780 & 0.785 \\
Yes & Yes & Yes & \textbf{0.811} & 0.813 & 0.809 & 0.823
\end{tabular}
\caption{IF-M on Marine vessel traffic data: performance quantified using ROC AUC for combinations of stream transformations and sub-stream length $D$. Best across all transformations are in bold.}
\label{tab:marine-path-performance-new-IF-M}
\end{table}

\begin{table}[h]
\centering
\scriptsize 
\setlength{\tabcolsep}{3pt} 
\begin{tabular}{ccc|cccc}
 \multicolumn{3}{c|}{\multirow{2}{*}{Transformation}} &
 \multicolumn{4}{c}{Conformance $\dist(\;\cdot \; ; \Sig^3(\mathcal C))$} \\
 \cline{4-7}
 & & & \multicolumn{4}{c}{Sub-stream length $D$} \\
 Lead-lag & Time-Diff & Inv.~Reset &
 $4$km & $8$km & $16$km & $32$km \\
 \hline
No & No & No    & 0.627 & 0.623 & 0.645 & 0.660 \\
No & No & Yes   & 0.686 & 0.701 & 0.715 & 0.718 \\
No & Yes & No   & 0.731 & 0.714 & 0.737 & 0.771 \\
No & Yes & Yes  & 0.777 & \textbf{0.784} & \textbf{0.789} & \textbf{0.808} \\
Yes & No & No   & 0.617 & 0.596 & 0.634 & 0.668 \\
Yes & No & Yes  & 0.692 & 0.701 & 0.691 & 0.713 \\
Yes & Yes & No  & 0.725 & 0.692 & 0.716 & 0.757 \\
Yes & Yes & Yes & \textbf{0.779} & 0.782 & 0.801 & 0.823
\end{tabular}
\caption{IF-S on Marine vessel traffic data: performance quantified using ROC AUC for combinations of stream transformations and sub-stream length $D$. Best across all transformations are in bold.}
\label{tab:marine-path-performance-updated-IF-S}
\end{table}

\begin{table}[h]
\centering
\scriptsize 
\setlength{\tabcolsep}{3pt} 
\begin{tabular}{ccc|cccc}
 \multicolumn{3}{c|}{\multirow{2}{*}{Transformation}} &
 \multicolumn{4}{c}{Conformance $\dist(\;\cdot \; ; \Sig^3(\mathcal C))$} \\
 \cline{4-7}
 & & & \multicolumn{4}{c}{Sub-stream length $D$} \\
 Lead-lag & Time-Diff & Inv.~Reset &
 $4$km & $8$km & $16$km & $32$km \\
 \hline
No & No & No    & 0.543 & 0.542 & 0.522 & 0.513 \\
No & No & Yes   & 0.555 & \textbf{0.585} & 0.564 & 0.535 \\
No & Yes & No   & 0.547 & 0.543 & 0.526 & 0.520 \\
No & Yes & Yes  & 0.559 & 0.589 & 0.572 & 0.545 \\
Yes & No & No   & 0.543 & 0.541 & 0.522 & 0.513 \\
Yes & No & Yes  & \textbf{0.566} & 0.556 & 0.542 & 0.533 \\
Yes & Yes & No  & 0.547 & 0.543 & 0.526 & 0.520 \\
Yes & Yes & Yes & 0.572 & 0.563 & \textbf{0.554} & \textbf{0.544} \\
\end{tabular}
\caption{LOF-M on Marine vessel traffic data: performance quantified using ROC AUC for combinations of stream transformations and sub-stream length $D$. Best across all transformations are in bold.}
\label{tab:marine-path-performance-updated-LOF-M}
\end{table}

\begin{table}[h]
\centering
\scriptsize 
\setlength{\tabcolsep}{3pt} 
\begin{tabular}{ccc|cccc}
 \multicolumn{3}{c|}{\multirow{2}{*}{Transformation}} &
 \multicolumn{4}{c}{Conformance $\dist(\;\cdot \; ; \Sig^3(\mathcal C))$} \\
 \cline{4-7}
 & & & \multicolumn{4}{c}{Sub-stream length $D$} \\
 Lead-lag & Time-Diff & Inv.~Reset &
 $4$km & $8$km & $16$km & $32$km \\
 \hline
No & No & No    & 0.484 & 0.500 & 0.491 & 0.492 \\
No & No & Yes   & 0.565 & 0.572 & 0.562 & 0.520 \\
No & Yes & No   & 0.511 & 0.505 & 0.493 & 0.494 \\
No & Yes & Yes  & 0.565 & 0.572 & 0.561 & 0.520 \\
Yes & No & No   & 0.484 & 0.500 & 0.491 & 0.493 \\
Yes & No & Yes  & 0.564 & 0.569 & 0.558 & 0.518 \\
Yes & Yes & No  & 0.530 & 0.533 & 0.553 & \textbf{0.600} \\
Yes & Yes & Yes & \textbf{0.574} & \textbf{0.588} & \textbf{0.589} & 0.584 \\
\end{tabular}
\caption{LOF-S on Marine vessel traffic data: performance quantified using ROC AUC for combinations of stream transformations and sub-stream length $D$. Best across all transformations are in bold.}
\label{tab:marine-path-performance-updated-LOF-S}
\end{table}

\subsubsection{Anomlous language detection}\label{subsec:language}
Next, we consider a sample of words. The corpus consists of English words, and the test set consists of words in one of the six languages: English, German, French, Italian, Polish, and Swedish. Words are coded into multivariate streams by taking one-hot encoding of the alphabets and cumulative sums. We obtained a corpus of 70,000 English words and a test set with 10,000 English words and 10,000 words from other five languages (2,000 each). We use a signature of $N=2$ on this test, due to the size of the dataset and the result is reported in \autoref{tab:language-performance}.

\begin{table}[h]
\centering
\scriptsize
\setlength{\tabcolsep}{3pt} 
\begin{tabular}{|l|c|c|c|c|c|}
\hline
 & SigMahaKNN & IF-M & IF-S & LOF-M & LOF-S \\ \hline
AUC & \textbf{0.878} & 0.713 & 0.723 & 0.769 & 0.787 \\ \hline
Standard Error & 0.002 & 0.004 & 0.004 & 0.003 & 0.003 \\ \hline
\end{tabular}
\caption{AUC and Standard Error values for different methods in the language dataset. Standard Error is based on bootstrapping of $10^4$ samples.}
\label{tab:language-performance}
\end{table}

\subsubsection{Univariate time series}\label{subsec:univariate}
For the specific case of detecting anomalous univariate time series, we benchmark our method against the ADSL shapelet method of Beggel et al.~\cite{beggel2019time}, using their set of 28 data sets from the UEA \& UCR time series repository \cite{bagnall2017repository} adapted in exactly the same manner.  Each data set comprises a set of time series of equal length, together with class labels.  One class (the same as in ADSL) is designated as a normal class, with all other classes designated as anomalies. To prepare the data for our method, we convert each time series into a 2-dimensional stream by incorporating a uniformly increasing time dimension.  We apply no other transformations to the data and take signatures of order $N=5$.

We create training and test sets exactly as in ADSL.  The training corpus $\mathcal{C}$ consists of 80\% of the normal time series, contaminated by a proportion of anomalies (we compute results for anomaly rates of 0.1\% and 5\%).  Across these data sets $|\mathcal{C}|$ ranges from 10 (Beef) to 840 (ChlorineConcentration at 5\%), $|\mathcal{I}_{\text{normal}}|$ ranges from 2 (Beef) to 200 (ChlorineConcentration), and $|\mathcal{I}_{\text{anomaly}}|$ ranges from 19 (BeetleFly and BirdChicken at 0.1\%) to 6401 (Wafer at 5\%).  We run experiments with ten random train-test splits, and take the median result.  The performance measure used by ADSL is the balanced accuracy, which requires a threshold to be set for detecting anomalies.  We report the best achievable balanced accuracy across all possible thresholds, and compare against the best value reported for ADSL.  Figure \ref{fig:time-series-benchmark} plots our results.

\begin{figure}[h]
    \centering
    \includegraphics[width=0.9\linewidth]{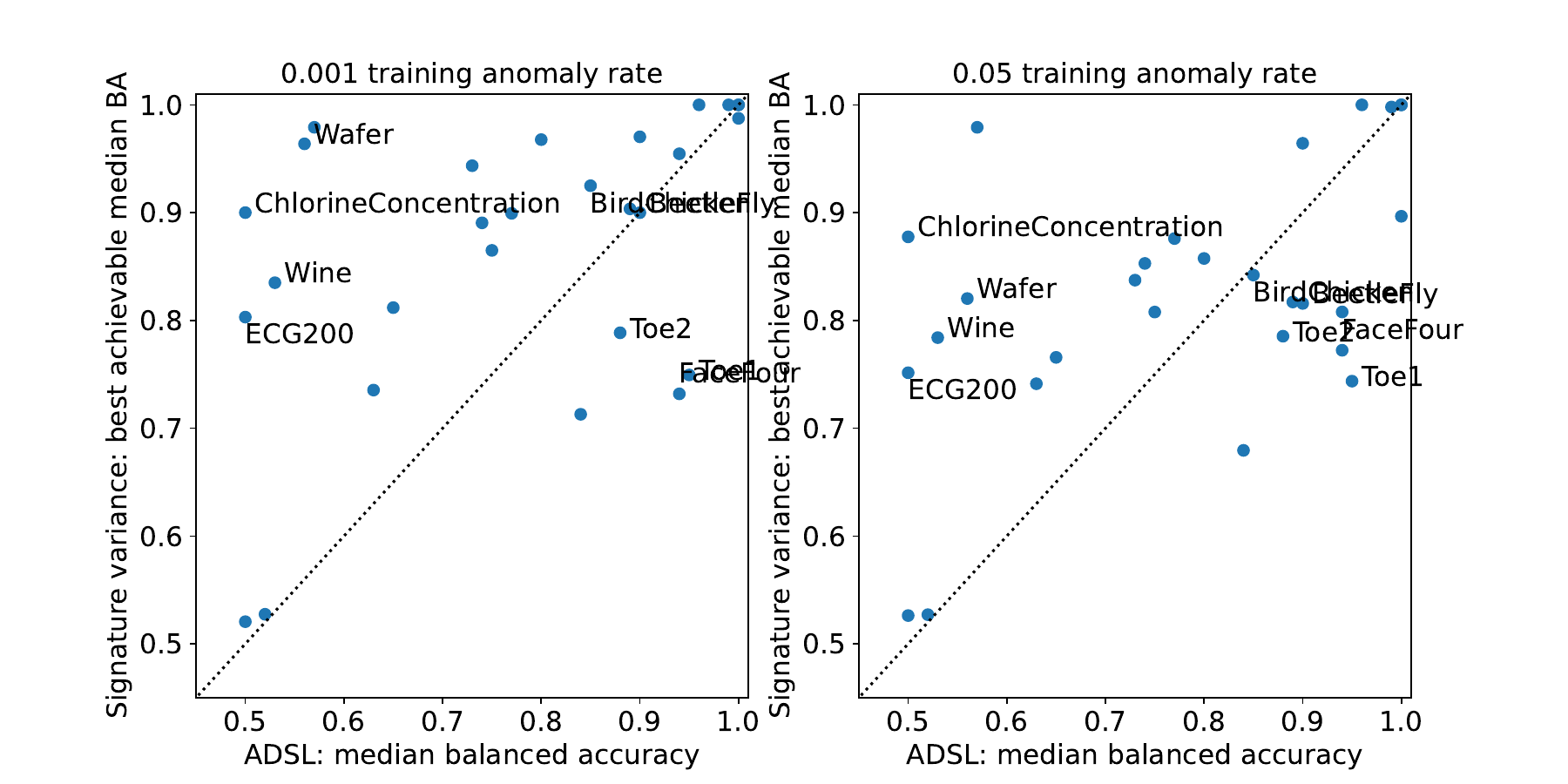}
    \caption{Comparison of our method against ADSL \cite{beggel2019time}}.
  \label{fig:time-series-benchmark}
\end{figure}

Our method performs competitively with ADSL, both when the proportion of anomalies in the training corpus is low and when it is high.  It is able to detect anomalies in four of the six data sets where ADSL struggles because the anomalies are less visually distinguishable (ChlorineConcentration, ECG200, Wafer, Wine).  However, there are data sets where ADSL performs better (BeetleFly, BirdChicken, FaceFour, ToeSegmentation1 and ToeSegmentation2): these data sets largely originate from research into shapelet methods, and they appear to contain features that are detected well by shapelets.  Applying transformations to the data sets before input may improve our method's results.

\section{Conclusion}

We proposed a generalisation of Mahalanobis distance, the variance norm, and showed that when combining with paths signatures and nearest-neighbours, SigMahaKNN has attractive theoretical properties such as reparametrisation invariance, concatenation invariance and graded discrimination power. We compared SigMahaKNN with isolation forest and shapelet-based methods for detecting anomalous streams with encouraging results.

\paragraph{Acknowledgement and funding disclosure}
The authors are grateful for the contribution of Imanol Perez Arribas and Jonathan H. 
Z. S. was supported by the EPSRC [EP/S026347/1]. 
R. C., and P. F. were supported by the Alan Turing Institute.
T. L. was funded in part by the EPSRC [EP/S026347/1], in part by The Alan Turing Institute under the EPSRC [EP/N510129/1], the Data Centric Engineering Programme (under the Lloyd’s Register Foundation grant G0095), the Defence and Security Programme (funded by the UK Government) and in part by the Hong Kong Innovation and Technology Commission (InnoHK Project CIMDA).

\bibliography{Reference.bib, ref.bib}        
\bibliographystyle{abbrv}  

\appendix

\section{Plots of conformance distances for PenDigits data set}
Here we provide the full distribution of conformance distance for the PenDigits dataset.
\begin{figure}[h!]
  \begin{subfigure}{.2\textwidth}
    \centering
    \includegraphics[width=1.0\linewidth]{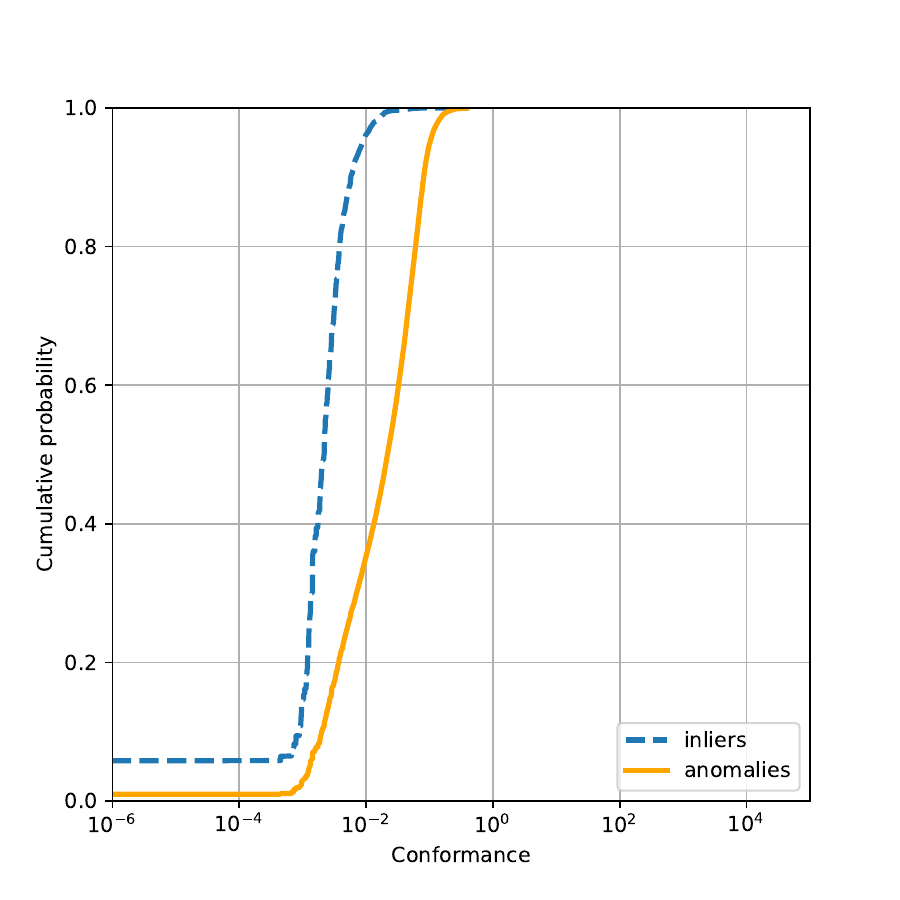}
    \caption{$N=1$}
  \end{subfigure}
  \begin{subfigure}{.2\textwidth}
    \centering
    \includegraphics[width=1.0\linewidth]{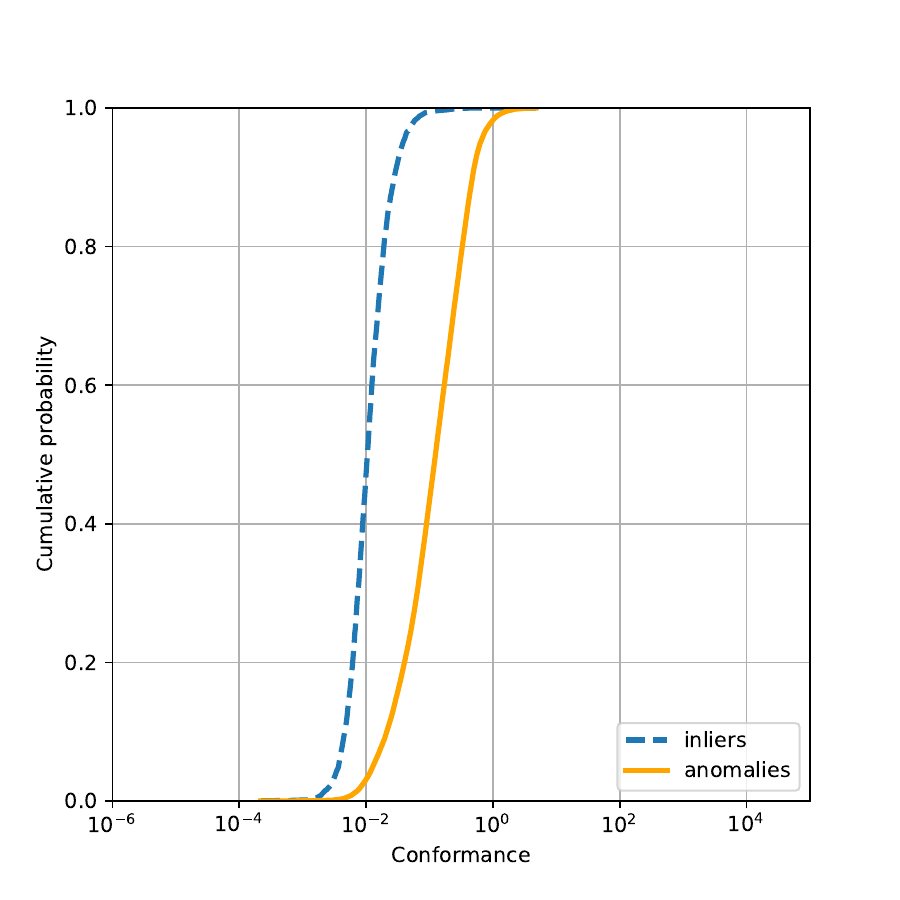}
    \caption{$N=2$}
  \end{subfigure}
  \newline
  \begin{subfigure}{.2\textwidth}
    \centering
    \includegraphics[width=1.0\linewidth]{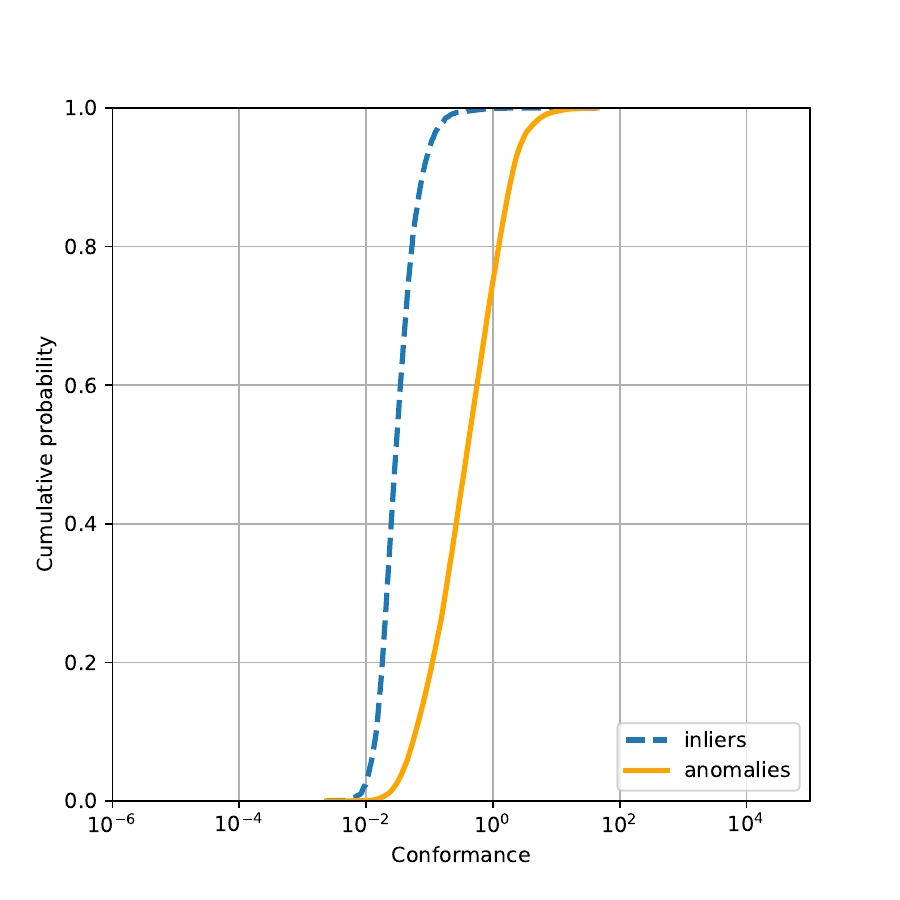}
    \caption{$N=3$}
  \end{subfigure}
  \begin{subfigure}{.2\textwidth}
    \centering
    \includegraphics[width=1.0\linewidth]{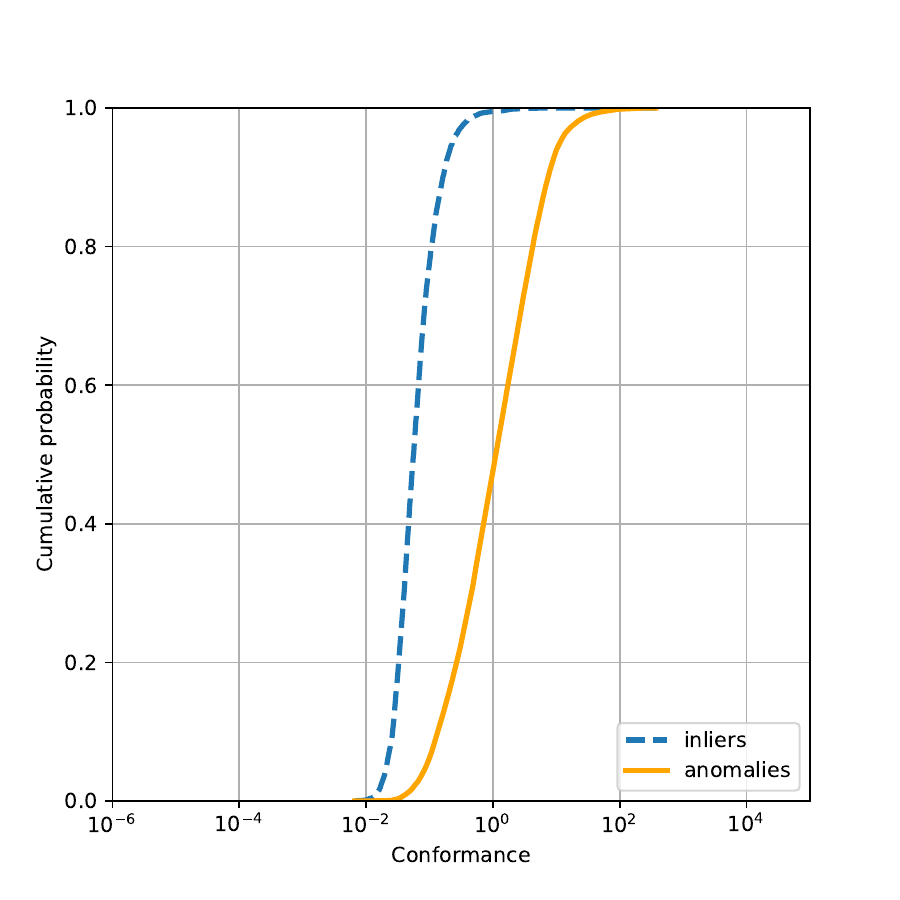}
    \caption{$N=4$}
  \end{subfigure}
  \newline
  \begin{subfigure}{0.2\textwidth}
    \centering
    \includegraphics[width=1.0\linewidth]{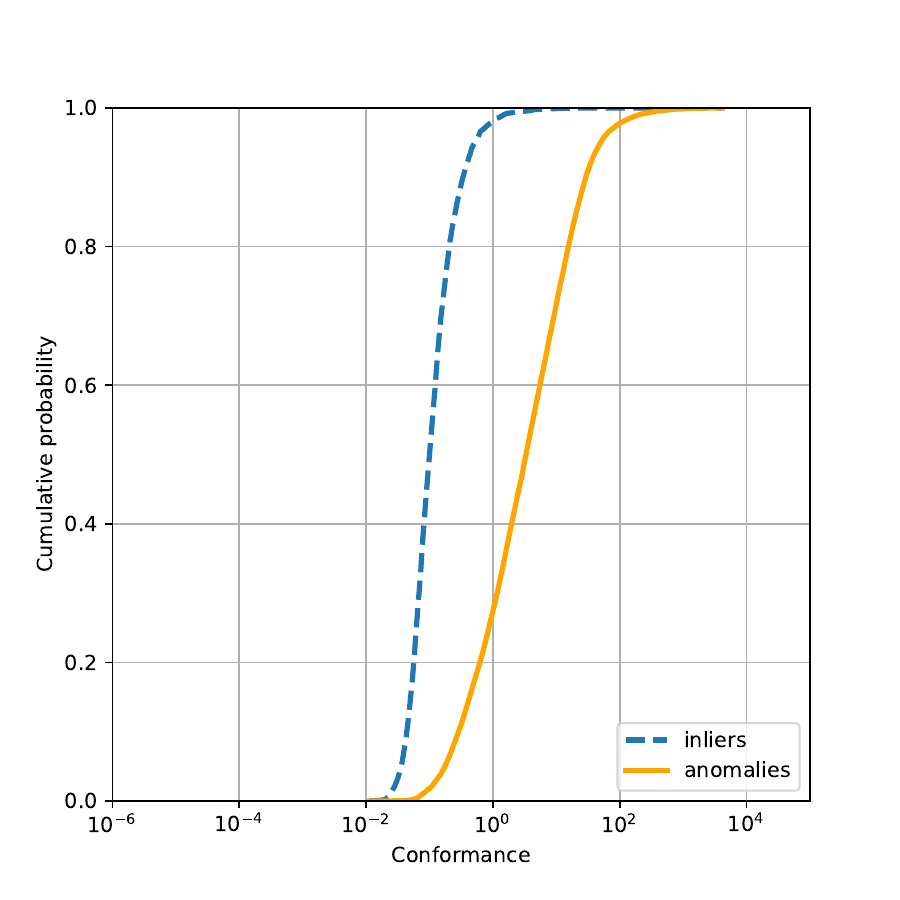}
    \caption{$N=5$}
  \end{subfigure}
  \caption{Empirical cumulative distributions of the conformance distance $\dist(\;\cdot \; ; \Sig^N(\mathcal C))$, obtained for normal and anomalous testing data and based on computing signatures of order $N$.}
  \label{fig:pendigits-conformances}
\end{figure}

\section{Table of Results for Univariate Time Series Data}
Here we provide complete AUC results for our univariate time series experiment.
\subsection{Table of Results for 0.1\% Anomaly Rate}
\begin{table}[h!]
  \scriptsize
  \centering
  \begin{tabular}{lll}
    \toprule
    Dataset & SigMahaKNN & ADSL \\
    \midrule
Adiac & \textbf{1.00 (0.00)} & 0.99 (0.10) \\
ArrowHead & \textbf{0.812 (0.071)} & 0.65 (0.03) \\
Beef & \textbf{0.979 (0.184)} & 0.57 (0.15) \\
BeetleFly & 0.90 (0.063) & \textbf{0.90 (0.08)} \\
BirdChicken & \textbf{0.925 (0.106)} & 0.85 (0.15) \\
CBF & \textbf{0.968 (0.016)} & 0.8 (0.04) \\
ChlorineConcentration & \textbf{0.900 (0.008)} & 0.5 (0.0) \\
Coffee & 0.713 (0.109) & \textbf{0.84 (0.04)} \\
ECG200 & \textbf{0.803 (0.068)} & 0.5 (0.03) \\
ECGFiveDays & \textbf{0.955 (0.015)} & 0.94 (0.11) \\
FaceFour & 0.732 (0.062) & \textbf{0.94 (0.10)} \\
GunPoint & \textbf{0.865 (0.061)} & 0.75 (0.03) \\
Ham & \textbf{0.521 (0.035)} & 0.5 (0.02) \\
Herring & \textbf{0.527 (0.062)} & 0.52 (0.02) \\
Lightning2 & 0.735 (0.076) & \textbf{0.63 (0.07)} \\
Lightning7 & \textbf{0.944 (0.091)} & 0.73 (0.11) \\
Meat & 0.988 (0.049) & \textbf{1.00 (0.04)} \\
MedicalImages & \textbf{0.970 (0.039)} & 0.9 (0.03) \\
MoteStrain & \textbf{0.891 (0.012)} & 0.74 (0.01) \\
Plane & \textbf{1.00 (0.00)} & \textbf{1.00 (0.04)} \\
Strawberry & \textbf{0.899 (0.008)} & 0.77 (0.03) \\
Symbols & \textbf{1.00 (0.007)} & 0.96 (0.02) \\
ToeSegmentation1 & 0.749 (0.039) & \textbf{0.95 (0.01)} \\
ToeSegmentation2 & 0.789 (0.052) & \textbf{0.88 (0.02)} \\
Trace & \textbf{1.00 (0.052)} & \textbf{1.00 (0.04)} \\
TwoLeadECG & \textbf{0.904 (0.015)} & 0.89 (0.01) \\
Wafer & \textbf{0.964 (0.012)} & 0.56 (0.02) \\
Wine & \textbf{0.835 (0.094)} & 0.53 (0.02) \\
    \bottomrule
  \end{tabular}
  \caption{Comparison of SigMahaKNN and ADSL Performances at 0.1\% Anomaly Rate}
  \label{table:0.1_percent_anomaly_rate}
\end{table}

\subsection{Table of Results for 5\% Anomaly Rate}
\begin{table}[h!]
  \scriptsize
  \centering
  \begin{tabular}{lll}
    \toprule
    Dataset & SigMahaKNN & ADSL \\
    \midrule
Adiac & \textbf{0.998 (0.133)} & 0.99 (0.10) \\
ArrowHead & \textbf{0.766 (0.081)} & 0.65 (0.03) \\
Beef & \textbf{0.979 (0.184)} & 0.57 (0.15) \\
BeetleFly & \textbf{0.816 (0.202)} & 0.9 (0.08) \\
BirdChicken & \textbf{0.842 (0.106)} & 0.85 (0.15) \\
CBF & \textbf{0.858 (0.034)} & 0.8 (0.04) \\
ChlorineConcentration & \textbf{0.878 (0.013)} & 0.5 (0.0) \\
Coffee & 0.679 (0.130) & \textbf{0.84 (0.04)} \\
ECG200 & \textbf{0.752 (0.074)} & 0.5 (0.03) \\
ECGFiveDays & 0.808 (0.018) & \textbf{0.94 (0.11)} \\
FaceFour & 0.772 (0.087) & \textbf{0.94 (0.10)} \\
GunPoint & \textbf{0.808 (0.070)} & 0.75 (0.03) \\
Ham & \textbf{0.526 (0.032)} & 0.5 (0.02) \\
Herring & \textbf{0.527 (0.065)} & 0.52 (0.02) \\
Lightning2 & \textbf{0.741 (0.065)} & 0.63 (0.07) \\
Lightning7 & \textbf{0.837 (0.091)} & 0.73 (0.11) \\
Meat & 0.897 (0.073) & \textbf{1.00 (0.04)} \\
MedicalImages & \textbf{0.964 (0.027)} & 0.9 (0.03) \\
MoteStrain & \textbf{0.853 (0.020)} & 0.74 (0.01) \\
Plane & \textbf{1.00 (0.036)} & \textbf{1.00 (0.04)} \\
Strawberry & \textbf{0.876 (0.026)} & 0.77 (0.03) \\
Symbols & \textbf{1.00 (0.021)} & 0.96 (0.02) \\
ToeSegmentation1 & 0.744 (0.037) & \textbf{0.95 (0.01)} \\
ToeSegmentation2 & 0.785 (0.055) & \textbf{0.88 (0.02)} \\
Trace & \textbf{1.00 (0.045)} & \textbf{1.00 (0.04)} \\
TwoLeadECG & 0.817 (0.021) & \textbf{0.89 (0.01)} \\
Wafer & \textbf{0.820 (0.029)} & 0.56 (0.02) \\
Wine & \textbf{0.784 (0.101)} & 0.53 (0.02) \\
    \bottomrule
  \end{tabular}
  \caption{Comparison of SigMahaKNN and ADSL Performances at 5\% Anomaly Rate}
  \label{table:5_percent_anomaly_rate}
\end{table}

\end{document}